\documentclass{article}

\usepackage{natbib}

\usepackage{textcomp}

\usepackage{hyperref}

\usepackage[accepted]{icml2011}

\usepackage{amssymb}
\usepackage{amsthm}
\usepackage{amsmath}

\newtheorem{theorem}{Theorem}
\newtheorem{lemma}{Lemma}

\icmltitlerunning{PAC-Bayesian Analysis of the Exploration-Exploitation Trade-off}

\begin{document} 

\twocolumn[
\icmltitle{PAC-Bayesian Analysis of the Exploration-Exploitation Trade-off}

\icmlauthor{Yevgeny Seldin}{seldin@tuebingen.mpg.de}
\icmladdress{Max Planck Institute for Intelligent Systems,
            T\"{u}bingen, Germany}
\icmlauthor{Nicol{\`o} Cesa-Bianchi}{nicolo.cesa-bianchi@unimi.it}
\icmladdress{Dipartimento di Scienze dell''Informazione,
            Universit{\`a} degli Studi di Milano, Italy}
\icmlauthor{Fran\c{c}ois Laviolette}{francois.laviolette@ift.ulaval.ca}
\icmladdress{Universit\'{e} Laval,
            Qu\'{e}bec, Canada}
\icmlauthor{Peter Auer}{auer@unileoben.ac.at}
\icmladdress{Chair for Information Technology,
            University of Leoben, Austria}
\icmlauthor{John Shawe-Taylor}{jst@cs.ucl.ac.uk}
\icmladdress{University College London,
            UK}
\icmlauthor{Jan Peters}{jan.peters@tuebingen.mpg.de}
\icmladdress{Max Planck Institute for Intelligent Systems,
            T\"{u}bingen, Germany}

\icmlkeywords{PAC-Bayes, Martingales, Multiarmed Bandits, Exploration-Exploitation}

\vskip 0.3in
]

\begin{abstract} 
We develop a coherent framework for integrative simultaneous analysis of the exploration-exploitation and model order selection trade-offs. We improve over our preceding results on the same subject \cite{SLST+11} by combining PAC-Bayesian analysis with Bernstein-type inequality for martingales. Such a combination is also of independent interest for studies of multiple simultaneously evolving martingales.
\end{abstract} 

\section{Introduction}

The trade-off between exploration and exploitation is a fundamental question in reinforcement learning. Model order selection, which is a trade-off between model complexity and its empirical data fit, is even a more basic question in machine learning. To the best of our knowledge, we develop the first framework that enables to consider these two trade-offs simultaneously from a finite sample perspective. The importance of simultaneous consideration of the two trade-offs can be illustrated by the following simple example. Imagine we have a web page, where we can show a visitor a single advertisement out of a pool of advertisements. Assume that we are given access to additional side information about the visitors, which we are allowed to use in our choice of advertisements (this is generally known as contextual bandits problem). Further, imagine that the amount of available (contextual) side information is very large (and potentially unlimited). Considering all side information from the beginning will result in an overcomplicated model that will take prohibitively many trials to learn. Instead, similar to supervised learning, we should start with a simple model and increase its complexity as our experience grows. However, unlike in supervised learning, we have to learn under limited feedback. This means that the model order selection trade-off has to be considered simultaneously with the exploration-exploitation trade-off. We develop an integrative framework that provides finite sample guarantees for both trade-offs simultaneously.

Our solution is based on extending PAC-Bayesian analysis of supervised learning with i.i.d. samples to problems with limited feedback and sequentially dependent samples. PAC-Bayesian analysis was introduced over a decade ago \cite{STW97,ST+98,McA98,See02} and has since made a significant contribution to the analysis and development of supervised learning methods. The power of PAC-Bayesian approach lies in successful marriage of the flexibility and intuitiveness of Bayesian models with the rigor of PAC analysis. PAC-Bayesian bounds provide an explicit and often intuitive and easy-to-optimize trade-off between model complexity and empirical data fit, where the complexity can be nailed down to the resolution of individual hypotheses via the prior definition. The PAC-Bayesian analysis was applied to derive generalization bounds and new algorithms for linear classifiers and maximum margin methods \cite{LST02, McA03b, GLLM09}, structured prediction \cite{McA07}, and clustering-based classification models \cite{ST10}, to name just a few. However, the application of PAC-Bayesian analysis beyond the supervised learning domain remained surprisingly limited. In fact, the only additional domain known to us is density estimation \cite{ST10, HST10}.

Some potential advantages of applying PAC-Bayesian analysis in reinforcement learning were recently pointed out by several researchers, including \citet{TP10} and \citet{FP10}. \citet{TP10} suggested that the mutual information between states and actions in a policy can be used as a natural regularizer in reinforcement learning. They showed that regularization by mutual information can be incorporated into Bellman equations and thereby computed efficiently. Tishby and Polani conjectured that PAC-Bayesian analysis can be applied to justify such form of regularization and provide generalization guarantees for it.

\citet{FP10} suggested a PAC-Bayesian analysis of batch reinforcement learning. However, batch reinforcement learning does not involve the exploration-exploitation trade-off.

One of the reasons for the difficulty of applying PAC-Bayesian analysis to address the exploration-exploitation trade-off is the limited feedback (the fact that we only observe the reward for the action taken, but not for all the rest). In supervised learning (and also in density estimation) the empirical error for each hypothesis within a hypotheses class can be evaluated on all the samples and therefore the size of the sample available for evaluation of all the hypotheses is the same (and usually relatively large). In the situation of limited feedback the sample from one action cannot be used to evaluate another action and the sample size of ``bad'' actions has to increase sublinearly in the number of game rounds. In \cite{SLST+11} we resolved this issue by applying weighted sampling strategy \cite{SB98}, which is commonly used in the analysis of non-stochastic bandits \cite{ACB+02}, but has not been applied to the analysis of stochastic bandits previously.

The usage of weighted sampling introduces two new difficulties. One is sequential dependence of the samples: the rewards we observe influence the distribution over actions we play and through this distribution influence the variance of the subsequent weighted sample variables. In \cite{SLST+11} we handled this dependence by combining PAC-Bayesian analysis with Hoeffding-Azuma-type inequalities for martingales. 

The second problem introduced by weighted sampling is the growing variance of the weighted sample variables. We did not succeed to take full control over the variance in \cite{SLST+11} and the bound we obtained there depended on $1/\varepsilon_t$, where $\varepsilon_t$ is the minimal probability for sampling any action at time step $t$. Here we improve this dependence to $1/\sqrt{\varepsilon_t}$ by combining PAC-Bayesian analysis with Bernstein-type inequality for martingales. This improvement enables to tighten the regret bounds from $O(K^{1/2} t^{3/4})$ to $O(K^{1/3} t^{2/3})$, where $K$ is the number of arms and $t$ is the game round. The combination PAC-Bayesian analysis with Bernstein-type inequality for martingales is also of independent interest for studies of multiple simultaneously evolving martingales.

At the end of Section \ref{sec:main} we suggest possible ways to tighten the analysis further to get $O(\sqrt{Kt})$ regret bounds. These further improvements will be studied in detail in future work.

We emphasize that although this paper is focused on the multiarmed bandit problem, our main goal is not improving existing bounds for stochastic multiarmed bandits, which are already tight up to $\sqrt{\ln(K)}$ factors \cite{AB09,AO10}, but rather developing a new powerful tool for reinforcement learning in domains with a richer structure. For example, \citet{BLL+10} suggested $O\left(\sqrt{Kt \ln(N/\delta)}\right)$ and $O\left(\sqrt{t(d\ln t - \ln \delta)}\right)$ regret bounds for learning with expert advice in the bandit setting, where $N$ is the number of experts (in case it is finite) and $d$ is the VC-dimension of the set of experts (in case it is infinite). We believe that PAC-Bayesian analysis should enable to replace $\ln(N)$ and $d$ factors with $KL(\rho\|\mu)$, where $\rho(h)$ is a distribution over experts played by the algorithm and $\mu(h)$ is a prior distribution over experts that, for example, can reflect their complexity, and $KL$ is the $KL$-divergence. Such an approach is much more flexible, since it allows individual treatment of different experts (or policies) via the prior definition $\mu$ and can be applied to both finite and infinite policy spaces (or expert sets). Our experience in supervised learning shows that PAC-Bayesian analysis is also handful for treating tree-shaped graphical models (since $KL$-divergence decomposes into sum of $KL$-s according to the tree structure). This property can also be useful for contextual bandits and other reinforcement learning problems.

The subsequent sections are organized as follows: Section \ref{sec:main} surveys the main results of the paper and Section \ref{sec:dis} discusses the results. All the proofs are provided in the appendix.

\section{Main Results}
\label{sec:main}

We start with a general concentration result for martingales, which is based on combination of PAC-Bayesian analysis with Bernstein-type inequality for martingales. We apply this result to derive an instantaneous (per-round) generalization bound for the multiarmed bandit problem. This result is in turn applied to derive an instantaneous regret bound for the multiarmed bandits.

\subsection{PAC-Bayes-Bernstein Inequality for Martingales}

In order to present our concentration result for martingales we need a few definitions. Let ${\cal H}$ be an index (or a hypothesis) space, possibly uncountably infinite. Let $\{X_1(h), X_2(h), ...\}$ be martingale difference sequences, meaning that $\mathbb E [X_t(h)|{\cal T}_{t-1}] = 0$, where ${\cal T}_t = \{X_\tau(h)\}_{\substack{1\leq \tau \leq t,\\h \in {\cal H}}}$ is a set of martingale differences observed up to time $t$. ($\{X_t(h)\}_{h \in {\cal H}}$ do not have to be independent, we only need that the requirement on the conditional expectation is satisfied.) Let $M_t(h) = \sum_{\tau = 1}^t X_\tau(h)$ be martingales. Let $V_t(h) = \sum_{\tau=1}^t \mathbb E [X_\tau(h)^2|{\cal T}_{\tau-1}]$ be cumulative variances of the martingales. For a distribution $\rho$ over ${\cal H}$ define $M_t(\rho) = \mathbb E_{\rho(h)} [M_t(h)]$ and $V_t(\rho) = \mathbb E_{\rho(h)} [V_t(h)]$.

\begin{theorem}[PAC-Bayes-Bernstein Inequality]
\label{thm:PAC-Bayes-Bernstein}
Assume that $|X_t(h)| \leq C$ for all $t$ and $h$. Let $\{\mu_1, \mu_2,...\}$ be a sequence of ``reference'' $($``prior''$)$ distributions over ${\cal H}$, such that $\mu_t$ is independent of ${\cal T}_t$ $($but can depend on $t$$)$. Let $\{\bar V_1, \bar V_2, ...\}$ be a sequence of arbitrary numbers, such that $\bar V_t$ is independent of ${\cal T}_t$ $($but can depend on $t$$)$ and satisfy$:$
\begin{equation}
\label{eq:technical}
\sqrt{\frac{L_t}{(e-2) \bar V_t}} \leq \frac{1}{C},
\end{equation}
where
\[
L_t = 2 \ln(t+1) + \ln \frac{2}{\delta}.
\]
Then for all possible distributions $\rho_t$ over ${\cal H}$ given $t$ and for all $t$ simultaneously$:$
\begin{equation}
\label{eq:PAC-Bayes-Bernstein}
|M_t(\rho_t)| \leq \sqrt{(e-2)} \left (
\begin{array}{l}
KL(\rho_t\|\mu_t)\sqrt{\frac{\bar V_t}{L_t}}\\
\quad + V_t(\rho_t) \sqrt{\frac{L_t}{\bar V_t}}
+ \sqrt{L_t \bar V_t}
\end{array}
\right ).
\end{equation}
\end{theorem}

\subsection{Application to the Multiarmed Bandit Problem}

In order to apply our result to the multiarmed bandit problem we need some more definitions. Let ${\cal A}$ be a set of actions (arms) of size $|{\cal A}| = K$ and let $a \in {\cal A}$ denote the actions. Denote by $R(a)$ the expected reward of action $a$. Let $\pi_t$ be a distribution over ${\cal A}$ that is played at round $t$ of the game. Let $\{A_1,A_2,...\}$ be the sequence of actions played independently at random according to $\{\pi_1,\pi_2,...\}$ respectively. Let $\{R_1,R_2,...\}$ be the sequence of observed rewards. Denote by ${\cal T}_t = \left \{\{A_1,..,A_t\},\{R_1,..,R_t\}\right\}$ the set of taken actions and observed rewards up to round $t$ (by definition ${\cal T}_{t-1} \subset {\cal T}_t$).

For $t\geq 1$ and $a \in \{1,..,K\}$ define a set of random variables $R_t^a$:
\[
R_t^a = \left \{ \begin{array}{cl}\frac{1}{\pi_t(a)}R_t,&\mbox{if}~A_t=a\\0,&\mbox{otherwise.}\end{array} \right .
\]
Define: 
\[
\hat R_t(a) = \frac{1}{t} \sum_{\tau=1}^t R_\tau^a.
\]
Observe that $\mathbb E \hat R_t(a) = R(a)$.

Let $a^*$ be the best action (the action with the highest expected reward, if there are multiple ``best'' actions pick any of them). Define:
\begin{align*}
\Delta(a) &= R(a^*) - R(a)\\
\hat \Delta_t(a) &= \hat R_t(a^*) - \hat R_t(a).
\end{align*}

Observe that $t \left (\hat \Delta_t(a) - \Delta(a) \right)$ form a martingale. Let
\[
W_t(a) = \sum_{\tau=1}^t \mathbb E[([R_\tau^{a^*} - R_\tau^a] - [R(a^*) - R(a)])^2 | {\cal T}_{\tau-1}]
\]
be the cumulative variance of this martingale. 

Let $\{\varepsilon_1, \varepsilon_2, ...\}$ be a decreasing sequence that satisfies $\varepsilon_t \leq \min_a \pi_t(a)$. In the appendix we prove the following upper bound on $W_t(a)$.
\begin{lemma}
\label{lem:V}
For all $a$$:$
\[
W_t(a) \leq \frac{2t}{\varepsilon_t}.
\]
\end{lemma}

For a distribution $\rho$ over ${\cal A}$ define $\Delta(\rho) = \mathbb E_{\rho(a)} [\Delta(a)]$ and $\hat \Delta_t(\rho) = \mathbb E_{\rho(a)} [\hat \Delta_t(a)]$. The following theorem follows immediately from Theorem \ref{thm:PAC-Bayes-Bernstein} and Lemma \ref{lem:V} by taking $\bar V_t = \frac{2t}{\varepsilon_t}$.
\begin{theorem}
\label{thm:PAC-Bayes-Delta}
For any sequence of sampling distributions $\{\pi_1,\pi_2,...\}$ that are bounded from below by a decreasing sequence $\{\varepsilon_1, \varepsilon_2, ...\}$ that satisfies
\begin{equation}
\label{eq:technical-epsilon}
\frac{L_t}{2(e-2)t} \leq \varepsilon_t,
\end{equation}
where $\pi_t$ can depend on ${\cal T}_{t-1}$, and for any sequence of ``reference'' distributions $\{\mu_1,\mu_2,...\}$ over ${\cal A}$, such that $\mu_t$ is independent of ${\cal T}_t$ $($but can depend on $t$$)$, for all possible distributions $\rho_t$ given $t$ and for all $t \geq 1$ simultaneously with probability greater than $1-\delta$$:$
\begin{equation}
\label{eq:PAC-Bayes-Delta-Bernstein}
\left |\Delta(\rho_t) - \hat \Delta_t(\rho_t) \right | \leq \sqrt{\frac{2(e-2)}{t\varepsilon_t}} \left (
\frac{KL(\rho_t\|\mu_t)}{\sqrt{L_t}} + 2\sqrt{L_t} \right ).
\end{equation}
\end{theorem}

Theorem \ref{thm:PAC-Bayes-Delta} provides an improvement over the corresponding Theorems 2 and 3 in \cite{SLST+11} by decreasing the dependence on $\varepsilon_t$ from $1/\varepsilon_t$ to $1/\sqrt{\varepsilon_t}$. This in turn allows to improve the regret bound, which is shown next.
\begin{theorem}
\label{thm:regret-delta}
For $t < K$ let $\pi_t(a) = \frac{1}{K}$ for all $a$. Let $\gamma_t = K^{-1/3}t^{1/3}\sqrt{\ln K}$ and $\varepsilon_t = K^{-2/3}t^{-1/3}$ and for $t \geq (K - 1)$ let
\begin{equation}
\pi_{t+1}(a) = \tilde \rho_t^{_{exp}}(a) = (1 - K \varepsilon_{t+1}) \rho_t^{_{exp}}(a) + \varepsilon_{t+1},
\label{eq:tilderho}
\end{equation}
where
\begin{equation}
\rho_t^{_{exp}}(a) = \frac{1}{Z(\rho_t^{_{exp}})} e^{\gamma_t \hat R_t(a)}
\label{eq:rho}
\end{equation}
and
\[
Z(\rho_t^{_{exp}}) = \sum_a e^{\gamma_t \hat R_t(a)}.
\]
Then for $t \geq \max \left \{K, K^{4(e-2)} \sqrt{\frac{\delta}{2}} \right \}$ and satisfying \eqref{eq:technical-epsilon} $($which means that $2 \ln(t+1) + \ln \frac{2}{\delta} \leq 2 (e-2) \left(\frac{t}{K}\right)^{2/3}$$)$ the per-round regret $R(a^*) - R(\tilde \rho_t^{_{exp}})$ is bounded by:
\[
R(a^*) - R(\tilde \rho_t^{_{exp}}) \leq \frac{K^{1/3}}{(t+1)^{1/3}}
\left (
\begin{array}{l}
(16(e-2) + 1) \sqrt{\ln K}\\
+ 2\sqrt{2(e-2)L_t} + 1
\end{array}
\right )
\]
with probability greater than $1-\delta$ for all rounds $t$ simultaneously. This translates into a total regret of $\tilde O(K^{1/3}t^{2/3})$ $($where $\tilde O$ hides logarithmic factors$)$.
\end{theorem}

Theorem \ref{thm:regret-delta} improves the dependence on $t$ and $K$ from $\tilde O(K^{1/2} t^{3/4})$ in \cite{SLST+11} to $\tilde O(K^{1/3} t^{2/3})$. This improvement is due to better concentration result in Theorem \ref{thm:PAC-Bayes-Delta} (which is based on Theorem \ref{thm:PAC-Bayes-Bernstein}).

We note that there is still room for improvement, which we believe will enable to achieve regret bounds of $\tilde O(\sqrt{Kt})$. The main source of looseness is the usage of the crude global upper bound $\frac{2t}{\varepsilon_t}$ on the cumulative variances that holds for any distribution $\rho_t$. It is possible to show that we play according to the distributions $\{\tilde \rho_1^{_{exp}}, .., \tilde \rho_t^{_{exp}}\}$, then for ``good'' actions $a$ (those for which $\Delta(a) \leq \frac{1}{\gamma_t}$) the cumulative variance $W_t(a)$ is bounded by $C K t$ for some constant $C$. If we could show that for ``bad'' actions $a$ (those for which $\Delta(a) > \frac{1}{\gamma_t}$) the probability $\rho_t^{_{exp}}$ of picking such actions is bounded by $C \varepsilon_t / K$, then the cumulative variance $W_t(\rho_t^{_{exp}})$ would be bounded by $C K t$. This is, in fact, true for ``very bad'' actions (those, for which $\Delta(a)$ is close to 1) and it is also possible to show that it holds for $\mu_t^{_{exp}}$ (and hence $W_t(\mu_t^{_{exp}}) \leq C K t$), but it does not hold for actions with $\Delta(a)$ close to $\frac{1}{\gamma_t}$. However, we can possibly show that for such actions $\rho_t^{_{exp}}(a) \leq C \varepsilon_t /K$ for most of the rounds ($1 - \varepsilon_t$ fraction should suffice) and then we will be able to achieve $\tilde O(\sqrt{Kt})$ regret. This research direction will be explored in more details in future work.

\section{Discussion}
\label{sec:dis}

We presented an improved PAC-Bayesian analysis of martingales that is based on combination of PAC-Bayesian bound with Bernstein-type inequality for martingales. The new bound enables to provide better finite sample generalization and regret guarantees for exploration-exploitation and model order selection trade-offs simultaneously. There are several important and fascinating research directions that take root at our result.

First, our concentration result for martingales can be of interest in any study of multiple simultaneously evolving and possibly interdependent martingales, especially when the number of martingales is uncountably infinite and standard union bounds cannot be applied. Just as an example, our result can be applied to derive new generalization bounds for active learning \cite{BDL09}.

Another important direction is to tighten Theorems \ref{thm:PAC-Bayes-Delta} and \ref{thm:regret-delta}, so that the regret bound will match state-of-the-art regret bounds obtained by alternative techniques. We believe that the ideas mentioned at the end of the previous section can make it possible.

Once we have a bound that matches state-of-the-art regret bounds we can extend the technique to richer problems with large or infinite number of states, such as contextual bandits \cite{BLL+10}, or large or infinite number of actions, such as Gaussian process bandits \cite{SKKS10}. Through definition of appropriate priors over hypothesis spaces, PAC-Bayesian approach should enable to obtain bounds that involve natural measures of model complexity, such as mutual information between states and actions in contextual bandits. Such a measure of model complexity is more flexible than plain number of experts or VC-dimension used in \cite{BLL+10} since it allows to differentiate between complexities of individual hypotheses. A similar analysis was already performed and proved successful in the context of co-clustering in supervised and unsupervised learning \cite{ST10}.

\appendix

\section{Proofs}

In this appendix we provide the proofs of Theorems \ref{thm:PAC-Bayes-Bernstein} and \ref{thm:regret-delta} and Lemma \ref{lem:V}.

\subsection{Proof of Theorem \ref{thm:PAC-Bayes-Bernstein}}

The proof of Theorem \ref{thm:PAC-Bayes-Bernstein} relies on the following two lemmas. The first one is a Bernstein-type inequality, see the proof of Theorem 1 in \cite{BLL+10} for a proof.
\begin{lemma}[Bernstein's inequality]
\label{lem:Bernstein}
Let $X_1,..,X_t$ be a martingale difference sequence $($meaning that $\mathbb E [X_\tau | X_1,..,X_{\tau-1}] = 0$ for all $\tau$$)$, such that $X_\tau \leq C$ for all $\tau$. Let $M_t = \sum_{\tau = 1}^t X_\tau$ be the corresponding martingale and $V_t = \sum_{\tau = 1}^t \mathbb E [X_\tau^2 | X_1,..,X_{\tau - 1}]$ be the cumulative variance of this martingale. Then for any fixed $\lambda \in [0,\frac{1}{C}]$$:$
\[
\mathbb E e^{\lambda M_t - (e - 2) \lambda^2 V_t} \leq 1.
\]
\end{lemma}

The second lemma originates in statistical physics and information theory \cite{DV75,DE97,Gra11} and forms the basis of PAC-Bayesian analysis. See \cite{Ban06} for a proof.
\begin{lemma}[Change of measure inequality]
\label{lem:PAC-Bayes}
For any measurable function $\phi(h)$ on ${\cal H}$ and any distributions $\mu(h)$ and $\rho(h)$ on ${\cal H}$, we have$:$
\[
\mathbb E_{\rho(h)}[\phi(h)] \leq KL(\rho\|\mu) + \ln \mathbb E_{\mu(h)}[ e^{\phi(h)}].
\]
\end{lemma}

Now we are ready to state the proof of Theorem \ref{thm:PAC-Bayes-Bernstein}.
\begin{proof}[Proof of Theorem \ref{thm:PAC-Bayes-Bernstein}]
Take $\phi(h) = \lambda_t M_t(h) - (e-2) \lambda_t^2 V_t(h)$ and $\delta_t = \frac{1}{t(t+1)} \delta \leq \frac{1}{(t+1)^2}\delta$. (It is well-known that $\sum_{t=1}^\infty \frac{1}{t(t+1)} = \sum_{t=1}^\infty \left (\frac{1}{t} - \frac{1}{t+1}\right ) = 1$.) Then the following holds for all $\rho_t$ and $t$ simultaneously with probability greater than $1 - \frac{\delta}{2}$:
\begin{align}
\lambda_t M_t&(\rho_t) - (e-2) \lambda_t^2 V_t(\rho_t)\notag\\
&=\mathbb E_{\rho_t(h)} [\lambda_t M_t(h) - (e-2) \lambda_t^2 V_t(h)]\label{eq:1}\\
&\leq KL(\rho_t\|\mu_t) + \ln \mathbb E_{\mu_t(h)} [e^{\lambda_t M_t(\mu_t) - (e-2) \lambda_t^2 V_t(\mu_t)}]\label{eq:2}\\
&\leq KL(\rho_t\|\mu_t) + 2 \ln(t+1) + \ln \frac{2}{\delta}\notag\\
&\quad + \ln \mathbb E_{{\cal T}_t} \mathbb E_{\mu_t(h)} [e^{\lambda_t M_t(h) - (e-2) \lambda_t^2 V_t(h)}]\label{eq:3}\\
&= KL(\rho_t\|\mu_t) + L_t\notag\\
&\quad + \ln \mathbb E_{\mu_t(h)} \mathbb E_{{\cal T}_t} [e^{\lambda_t M_t(h) - (e-2) \lambda_t^2 V_t(h)}]\label{eq:4}\\
&\leq KL(\rho_t\|\mu_t) + L_t,\label{eq:5}
\end{align}
where \eqref{eq:1} is by definition of $M_t(\rho_t)$ and $V_t(\rho_t)$, \eqref{eq:2} is  by Lemma \ref{lem:PAC-Bayes}, \eqref{eq:3} holds with probability greater than $1-\frac{\delta}{2}$ by Markov's inequality and a union bound over $t$, \eqref{eq:4} is due to the fact that $\mu_t$ is independent of ${\cal T}_t$ and by definition of $L_t$, and \eqref{eq:5} is by Lemma  \ref{lem:Bernstein}.

By applying the same argument to martingales $-M_t(h)$ and taking a union bound over the two we obtain that with probability greater than $1-\delta$:
\begin{equation}
|M_t(\rho_t)| \leq \frac{KL(\rho_t\|\mu_t) + (e-2) \lambda_t^2 V_t(\rho_t) + L_t}{\lambda_t}.
\label{eq:MM}
\end{equation}
By taking
\[
\lambda_t = \sqrt{\frac{L_t}{(e-2)\bar V_t}}
\]
and substituting into \eqref{eq:MM} we obtain \eqref{eq:PAC-Bayes-Bernstein}. The technical condition \eqref{eq:technical} follows from the requirement that $\lambda_t \in [0,\frac{1}{C}]$.
\end{proof}

\subsection{Proof of Lemma \ref{lem:V}}

\begin{proof}[Proof of Lemma \ref{lem:V}]
\begin{align}
W_t(a) &= \sum_{\tau=1}^t \mathbb E[([R_\tau^{a^*} - R_\tau^a] - [R(a^*) - R(a)])^2 | {\cal T}_{\tau-1}]\notag\\
&= \left (\sum_{\tau=1}^t \mathbb E[(R_\tau^{a^*} - R_\tau^a)^2 | {\cal T}_{\tau-1}]\right ) - t\Delta(a)^2 \label{eq:61}\\
&\leq \left (\sum_{\tau=1}^t \left (\frac{\pi_\tau(a)}{\pi_\tau(a)^2} + \frac{\pi_\tau(a^*)}{\pi_\tau(a^*)^2} \right ) \right ) - t \Delta(a)^2\label{eq:6}\\
&= \left (\sum_{\tau=1}^t \left (\frac{1}{\pi_\tau(a)} + \frac{1}{\pi_\tau(a^*)} \right ) \right ) - t \Delta(a)^2\notag\\
&\leq \frac{2t}{\varepsilon_t},\label{eq:66}
\end{align}
where \eqref{eq:61} is due to the fact that $\mathbb E[R_\tau^a | {\cal T}_{\tau -1}] = R(a)$, \eqref{eq:6} is due to the fact that $R_t \leq 1$ and \eqref{eq:66} is due to the fact that $\frac{1}{\pi_\tau(a)} \leq \frac{1}{\varepsilon_t}$ for all $a$ and $1 \leq \tau \leq t$.
\end{proof}

\subsection{Proof of Theorem \ref{thm:regret-delta}}

\begin{proof}[Proof of Theorem \ref{thm:regret-delta}]
We take the same prior $\mu_t(a)$ that was used in \cite{SLST+11}
\begin{equation}
\mu_t^{_{exp}}(a) = \frac{1}{Z(\mu_t^{_{exp}})} e^{\gamma_t R(a)},
\label{eq:mu}
\end{equation}
where $Z(\mu_t^{_{exp}}) = \sum_a e^{\gamma_t R(a)}$ is the normalization factor.

We reuse the same regret decomposition we had in \cite{SLST+11}, but write it in a new form using $\Delta$-s:
\begin{align}
\Delta(\tilde \rho_t^{_{exp}}) &= \Delta(\rho_t^{_{exp}}) + [R(\rho_t^{_{exp}}) - R(\tilde \rho_t^{_{exp}})]\notag\\
&\leq [\Delta(\rho_t^{_{exp}}) - \hat \Delta_t(\rho_t^{_{exp}})] + \hat \Delta_t(\rho_t^{_{exp}}) + K \varepsilon_{t+1}\label{eq:7}\\
&\leq [\Delta(\rho_t^{_{exp}}) - \hat \Delta_t(\rho_t^{_{exp}})] + \frac{\ln K}{\gamma_t} + K \varepsilon_{t+1},\label{eq:regret-delta}
\end{align}
where in \eqref{eq:7} we used the bound on $[R(\rho_t^{_{exp}}) - R(\tilde \rho_t^{_{exp}})]$ obtained in \cite{SLST+11} and in \eqref{eq:regret-delta} we used Lemma \ref{lem:expsum} given below. Note that due to working with $\Delta$-s we are left to bound only one term instead of two terms we had to bound in \cite{SLST+11}.

\begin{lemma}
\label{lem:expsum}
Let $x_1 = 0$ and $x_2,..,x_n$ be $n-1$ arbitrary numbers. For any $\alpha > 0$ and $n \geq 2$$:$
\begin{equation}
\frac{\sum_{i=1}^n x_i e^{-\alpha x_i}}{\sum_{j=1}^n e^{-\alpha x_j}} \leq \frac{\ln(n)}{\alpha}.
\label{eq:expsum}
\end{equation}
\end{lemma}

\begin{proof}

Since negative $x_i$-s only decrease the left hand side of \eqref{eq:expsum} we can assume without loss of generality that all $x_i$-s are positive. Due to symmetry, the maximum is achieved when all $x_i$-s (except $x_1$) are equal:
\begin{equation}
\label{eq:x}
\frac{\sum_{i=1}^n x_i e^{-\alpha x_i}}{\sum_{j=1}^n e^{-\alpha x_j}} \leq \max_x \frac{(n-1) x e^{-\alpha x}}{1 + (n-1) e^{-\alpha x}}.
\end{equation}

We apply change of variables $y = e^{-\alpha x}$, which means that $x = \frac{1}{\alpha}\ln \frac{1}{y}$. By substituting this into the right hand side of \eqref{eq:x} we get
\[
\frac{1}{\alpha} \cdot \frac{(n-1) y \ln \frac{1}{y}}{1 + (n-1)y}.
\]
In order to prove the bound we have to show that $\frac{(n-1) y \ln \frac{1}{y}}{1 + (n-1)y} \leq \ln n$. 

By taking Taylor expansion of $\ln z$ around $z = n$ we have:
\[
\ln z \leq \ln n + \frac{1}{n} (z - n) = \ln n + \frac{z}{n} - 1.
\]
Thus:
\begin{align*}
\frac{(n-1) y \ln \frac{1}{y}}{1 + (n-1)y} &\leq \frac{(n-1) y (\ln n + \frac{1}{n y} - 1)}{1 + (n-1) y}\\
&\leq \frac{y (n-1) \ln n + \frac{n-1}{n}}{(n-1) y + 1}\\
&\leq \frac{(y (n-1) + 1) \ln n}{y(n-1) + 1} = \ln n,
\end{align*}
where the last inequality follows from the fact that $\frac{n-1}{n} \leq \ln n$ for $n \geq 2$.
\end{proof}

In order to obtain an explicit bound on $[\Delta(\rho_t) - \hat \Delta_t(\rho_t)]$ we need an explicit bound on $KL(\rho_t^{_{exp}}\|\mu_t^{_{exp}})$. To obtain such a bound we modify the procedure that was used in \cite{SLST+11}, which in turn was based on the procedure developed by \citet{LLST10}. Due to tighter concentration inequality in Theorem \ref{thm:PAC-Bayes-Bernstein} we obtain a tighter bound on $KL(\rho_t^{_{exp}}\|\mu_t^{_{exp}})$.

The derivation procedure starts with the following lemma, which is proved similarly to Lemma 12 in \cite{SLST+11}.
\begin{lemma}
\label{lem:KL}
For $\mu_t^{_{exp}}$ and $\rho_t^{_{exp}}$ defined by \eqref{eq:mu} and \eqref{eq:rho}$:$
\[
KL(\rho_t^{_{exp}}\|\mu_t^{_{exp}}) \leq \gamma_t \left (
\begin{array}{l}
[\Delta(\rho_t^{_{exp}}) - \hat \Delta_t(\rho_t^{_{exp}})]\\
\quad + [\hat \Delta_t(\mu_t^{_{exp}}) - \Delta(\mu_t^{_{exp}})] 
\end{array}
\right ).
\]
\end{lemma}

\begin{proof}
We use the following definitions:
\begin{align*}
Z'(\mu_t^{_{exp}}) &= \sum_a e^{-\gamma_t \Delta(a)}\\
&= \sum_a e^{-\gamma_t (R(a^*) - R(a))}\\
&= e^{-\gamma_t R(a^*)} Z(\mu_t^{_{exp}}).
\end{align*}
\begin{align*}
Z'(\rho_t^{_{exp}}) &= \sum_a e^{-\gamma_t \hat \Delta_t(a)}\\
&= \sum_a e^{-\gamma_t (\hat R_t(a^*) - \hat R_t(a))}\\
&= e^{-\gamma_t \hat R_t(a^*)} Z(\rho_t^{_{exp}}).
\end{align*}
The following identity is easily verified from the definitions:
\begin{align*}
\frac{1}{Z'(\mu_t^{_{exp}})} &= \frac{1}{Z(\mu_t^{_{exp}})} e^{\gamma_t R(a^*)}\\
&= \mu_t(a) e^{-\gamma_t R(a)} e^{\gamma_t R(a^*)}\\
&= \mu_t(a) e^{\gamma_t\Delta(a)}.
\end{align*}

Now we have:
\begin{align*}
&KL(\rho_t^{_{exp}}\|\mu_t^{_{exp}}) = \sum_a \rho_t(a) \ln \frac{e^{\gamma_t \hat R_t(a)} Z(\mu_t^{_{exp}})}{e^{\gamma_t R(a)} Z(\rho_t^{_{exp}})}\\
&=\sum_a \rho_t(a) \ln \frac{e^{-\gamma_t \hat \Delta_t(a)} Z'(\mu_t^{_{exp}})}{e^{-\gamma_t \Delta(a)} Z'(\rho_t^{_{exp}})}\\
&= \gamma_t [\Delta(\rho_t^{_{exp}}) - \hat \Delta_t(\rho_t^{_{exp}})] - \ln \frac{\sum_a e^{-\gamma_t \hat \Delta_t(a)}}{Z'(\mu_t^{_{exp}})}\\
&= \gamma_t [\Delta(\rho_t^{_{exp}}) - \hat \Delta_t(\rho_t^{_{exp}})] - \ln \sum_a \mu_t^{_{exp}}(a) e^{\gamma_t (\Delta(a) - \hat \Delta_t(a))}\\
&\leq \gamma_t \left ([\Delta(\rho_t^{_{exp}}) - \hat \Delta_t(\rho_t^{_{exp}})] + [\hat \Delta_t(\mu_t^{_{exp}}) - \Delta(\mu_t^{_{exp}})] \right).
\end{align*}
\end{proof}

Now we want to get an explicit upper bound on $KL(\rho_t^{_{exp}}\|\mu_t^{_{exp}})$. Note that for our choice of $\varepsilon_t$ the technical condition \eqref{eq:technical-epsilon} of Theorem \ref{thm:PAC-Bayes-Delta} is satisfied by $t$ large enough, so that 
\[
2 \ln(t+1) + \ln \frac{2}{\delta} \leq 2 (e-2) \left(\frac{t}{K}\right)^{2/3}.
\]
(This requirement is satisfied by $t = O\left(K \left(\ln \frac{1}{\delta} \right)^{3/2}\right)$.) By Theorem \ref{thm:PAC-Bayes-Delta} with probability greater than $1 - \delta$:
\begin{align}
\Delta(\rho_t^{_{exp}}) &- \hat \Delta_t(\rho_t^{_{exp}})\notag \\
&\leq \sqrt{\frac{2(e-2)}{t \varepsilon_t}} \left ( \frac{KL(\rho_t^{_{exp}}\|\mu_t^{_{exp}})}{\sqrt{L_t}} + 2 \sqrt{L_t} \right )\label{eq:D}
\end{align}
and
\[
\hat \Delta_t(\mu_t^{_{exp}}) - \Delta(\mu_t^{_{exp}}) \leq 2\sqrt{\frac{2(e-2)L_t}{t \varepsilon_t}}.
\]

By substituting this into Lemma \ref{lem:KL} we obtain:
\begin{align*}
KL&(\rho_t^{_{exp}}\|\mu_t^{_{exp}})\\
& \leq \gamma_t \sqrt{\frac{2(e-2)}{t \varepsilon_t}} \left ( \frac{KL(\rho_t^{_{exp}}\|\mu_t^{_{exp}})}{\sqrt{L_t}} + 4 \sqrt{L_t}\right).
\end{align*}
By reorganizing the terms:
\begin{equation}
KL(\rho_t^{_{exp}}\|\mu_t^{_{exp}}) \left (1 - \gamma_t \sqrt{\frac{2 (e-2)}{t \varepsilon_t L_t}} \right ) \leq 4 \gamma_t \sqrt{\frac{2(e-2)L_t}{t \varepsilon_t}}.
\label{eq:KLbound}
\end{equation}
Note that for our choice of $\gamma_t$ and $\varepsilon_t$:
\[
\gamma_t \sqrt{\frac{2(e-2)}{t \varepsilon_t L_t}} = \sqrt{\frac{2(e-2)K}{2 \ln(t+1) + \ln \frac{2}{\delta}}}.
\]
By simple algebraic manipulations we obtain that
\begin{equation}
\gamma_t \sqrt{\frac{2(e-2)}{t \varepsilon_t L_t}} \leq \frac{1}{2}
\label{eq:1-2}
\end{equation}
for
\[
t \geq K^{4(e-2)} \sqrt{\frac{\delta}{2}}.
\]

By substituting \eqref{eq:1-2} into \eqref{eq:KLbound} we obtain that:
\[
KL(\rho_t^{_{exp}}\|\mu_t^{_{exp}}) \leq 8 \gamma_t \sqrt{\frac{2(e-2)L_t}{t \varepsilon_t}}.
\]

By substituting this into \eqref{eq:D} we obtain
\begin{align*}
\Delta(\rho_t^{_{exp}}) &- \hat \Delta_t(\rho_t^{_{exp}}) \\
&\leq \sqrt{\frac{2(e-2)}{t \varepsilon_t}} \left (8 \gamma_t \sqrt{\frac{2(e-2)}{t \varepsilon_t}} + 2 \sqrt{L_t} \right ).
\end{align*}
For our choice of $\gamma_t$ and $\varepsilon_t$:
\[
\Delta(\rho_t^{_{exp}}) - \hat \Delta_t(\rho_t^{_{exp}}) \leq \frac{K^{1/3}}{t^{1/3}} \left ( 
\begin{array}{l}
16(e-2)\sqrt{\ln K}\\
\quad + 2\sqrt{2(e-2)L_t} 
\end{array}
\right)
\]
Substitution of the result into \eqref{eq:regret-delta} concludes the proof.
\end{proof}

\section*{Acknowledgments}

We thank Ronald Ortner for useful discussions. 

This work was supported in part by the IST Programme of the European Community, under the PASCAL2 Network of Excellence, IST-2007-216886, and by the European Community's Seventh Framework Programme (FP7/2007-2013), under grant agreement \textnumero 231495. This publication only reflects the authors' views.

\bibliography{bibliography}

\begin{thebibliography}{25}
\providecommand{\natexlab}[1]{#1}
\providecommand{\url}[1]{\texttt{#1}}
\expandafter\ifx\csname urlstyle\endcsname\relax
  \providecommand{\doi}[1]{doi: #1}\else
  \providecommand{\doi}{doi: \begingroup \urlstyle{rm}\Url}\fi

\bibitem[Audibert \& Bubeck(2009)Audibert and Bubeck]{AB09}
Audibert, Jean-Yves and Bubeck, S\'{e}bastien.
\newblock Minimax policies for adversarial and stochastic bandits.
\newblock In \emph{Proceedings of the International Conference on Computational
  Learning Theory (COLT)}, 2009.

\bibitem[Auer \& Ortner(2010)Auer and Ortner]{AO10}
Auer, Peter and Ortner, Ronald.
\newblock {UCB} revisited: Improved regret bounds for the stochastic
  multi-armed bandit problem.
\newblock \emph{Periodica Mathematica Hungarica}, 61\penalty0 (1-2):\penalty0
  55--65, 2010.

\bibitem[Auer et~al.(2002)Auer, Cesa-Bianchi, Freund, and Schapire]{ACB+02}
Auer, Peter, Cesa-Bianchi, Nicol{\`o}, Freund, Yoav, and Schapire, Robert~E.
\newblock The nonstochastic multiarmed bandit problem.
\newblock \emph{SIAM Journal of Computing}, 32\penalty0 (1), 2002.

\bibitem[Banerjee(2006)]{Ban06}
Banerjee, Arindam.
\newblock On {B}ayesian bounds.
\newblock In \emph{Proceedings of the International Conference on Machine
  Learning (ICML)}, 2006.

\bibitem[Beygelzimer et~al.(2009)Beygelzimer, Dasgupta, and Langford]{BDL09}
Beygelzimer, Alina, Dasgupta, Sanjoy, and Langford, John.
\newblock Importance weighted active learning.
\newblock In \emph{Proceedings of the International Conference on Machine
  Learning (ICML)}, 2009.

\bibitem[Beygelzimer et~al.(2010)Beygelzimer, Langford, Li, Reyzin, and
  Schapire]{BLL+10}
Beygelzimer, Alina, Langford, John, Li, Lihong, Reyzin, Lev, and Schapire,
  Robert~E.
\newblock Contextual bandit algorithms with supervised learning guarantees.
\newblock http://arxiv.org/abs/1002.4058, 2010.

\bibitem[Donsker \& Varadhan(1975)Donsker and Varadhan]{DV75}
Donsker, Monroe~D. and Varadhan, S.R.~Srinivasa.
\newblock Asymptotic evaluation of certain {Markov} process expectations for
  large time.
\newblock \emph{Communications on Pure and Applied Mathematics}, 28, 1975.

\bibitem[Dupuis \& Ellis(1997)Dupuis and Ellis]{DE97}
Dupuis, Paul and Ellis, Richard~S.
\newblock \emph{A Weak Convergence Approach to the Theory of Large Deviations}.
\newblock Wiley-Interscience, 1997.

\bibitem[Fard \& Pineau(2010)Fard and Pineau]{FP10}
Fard, Mahdi~Milani and Pineau, Joelle.
\newblock {PAC-Bayesian} model selection for reinforcement learning.
\newblock In \emph{Advances in Neural Information Processing Systems (NIPS)},
  2010.

\bibitem[Germain et~al.(2009)Germain, Lacasse, Laviolette, and
  Marchand]{GLLM09}
Germain, Pascal, Lacasse, Alexandre, Laviolette, Fran{\c c}ois, and Marchand,
  Mario.
\newblock {PAC-B}ayesian learning of linear classifiers.
\newblock In \emph{Proceedings of the International Conference on Machine
  Learning (ICML)}, 2009.

\bibitem[Gray(2011)]{Gra11}
Gray, Robert~M.
\newblock \emph{Entropy and Information Theory}.
\newblock Springer, 2 edition, 2011.

\bibitem[Higgs \& Shawe-Taylor(2010)Higgs and Shawe-Taylor]{HST10}
Higgs, Matthew and Shawe-Taylor, John.
\newblock A {PAC-Bayes} bound for tailored density estimation.
\newblock In \emph{Proceedings of the International Conference on Algorithmic
  Learning Theory (ALT)}, 2010.

\bibitem[Langford \& Shawe-Taylor(2002)Langford and Shawe-Taylor]{LST02}
Langford, John and Shawe-Taylor, John.
\newblock {PAC-Bayes} \& margins.
\newblock In \emph{Advances in Neural Information Processing Systems (NIPS)},
  2002.

\bibitem[Lever et~al.(2010)Lever, Laviolette, and Shawe-Taylor]{LLST10}
Lever, Guy, Laviolette, Fran\c{c}ois, and Shawe-Taylor, John.
\newblock Distribution-dependent {PAC-Bayes} priors.
\newblock In \emph{Proceedings of the International Conference on Algorithmic
  Learning Theory (ALT)}, 2010.

\bibitem[McAllester(1998)]{McA98}
McAllester, David.
\newblock Some {PAC-B}ayesian theorems.
\newblock In \emph{Proceedings of the International Conference on Computational
  Learning Theory (COLT)}, 1998.

\bibitem[McAllester(2003)]{McA03b}
McAllester, David.
\newblock Simplified {PAC-Bayesian} margin bounds.
\newblock In \emph{Proceedings of the International Conference on Computational
  Learning Theory (COLT)}, 2003.

\bibitem[McAllester(2007)]{McA07}
McAllester, David.
\newblock Generalization bounds and consistency for structured labeling.
\newblock In Bakir, G\"{o}khan, Hofmann, Thomas, Sch\"{o}lkopf, Bernhard,
  Smola, Alexander, Taskar, Ben, and Vishwanathan, S.V.N. (eds.),
  \emph{Predicting Structured Data}. The MIT Press, 2007.

\bibitem[Seeger(2002)]{See02}
Seeger, Matthias.
\newblock {PAC-Bayesian} generalization error bounds for {Gaussian} process
  classification.
\newblock \emph{Journal of Machine Learning Research}, 2002.

\bibitem[Seldin \& Tishby(2010)Seldin and Tishby]{ST10}
Seldin, Yevgeny and Tishby, Naftali.
\newblock {PAC-Bayesian} analysis of co-clustering and beyond.
\newblock \emph{Journal of Machine Learning Research}, 11, 2010.

\bibitem[Seldin et~al.(2011)Seldin, Laviolette, Shawe-Taylor, Peters, and
  Auer]{SLST+11}
Seldin, Yevgeny, Laviolette, Fran\c{c}ois, Shawe-Taylor, John, Peters, Jan, and
  Auer, Peter.
\newblock {PAC-Bayesian} analysis of martingales and multiarmed bandits.
\newblock http://arxiv.org/abs/1105.2416, 2011.

\bibitem[Shawe-Taylor \& Williamson(1997)Shawe-Taylor and Williamson]{STW97}
Shawe-Taylor, John and Williamson, Robert~C.
\newblock A {PAC} analysis of a {Bayesian} estimator.
\newblock In \emph{Proceedings of the International Conference on Computational
  Learning Theory (COLT)}, 1997.

\bibitem[Shawe-Taylor et~al.(1998)Shawe-Taylor, Bartlett, Williamson, and
  Anthony]{ST+98}
Shawe-Taylor, John, Bartlett, Peter~L., Williamson, Robert~C., and Anthony,
  Martin.
\newblock Structural risk minimization over data-dependent hierarchies.
\newblock \emph{IEEE Transactions on Information Theory}, 44\penalty0 (5),
  1998.

\bibitem[Srinivas et~al.(2010)Srinivas, Krause, Kakade, and Seeger]{SKKS10}
Srinivas, Niranjan, Krause, Andreas, Kakade, Sham~M., and Seeger, Matthias.
\newblock Gaussian process optimization in the bandit setting: No regret and
  experimental design.
\newblock In \emph{Proceedings of the International Conference on Machine
  Learning (ICML)}, 2010.

\bibitem[Sutton \& Barto(1998)Sutton and Barto]{SB98}
Sutton, Richard~S. and Barto, Andrew~G.
\newblock \emph{Reinforcement Learning: An Introduction}.
\newblock MIT Press, 1998.

\bibitem[Tishby \& Polani(2010)Tishby and Polani]{TP10}
Tishby, Naftali and Polani, Daniel.
\newblock Information theory of decisions and actions.
\newblock In Cutsuridis, Vassilis, Hussain, Amir, Taylor, John~G., and Polani,
  Daniel (eds.), \emph{Perception-Reason-Action Cycle: Models, Algorithms and
  Systems}. Springer, 2010.

\end{thebibliography}
\bibliographystyle{icml2011}

\end{document}